\documentclass[final,12pt]{colt2021}

\usepackage{hyperref}
\usepackage{enumerate}
\usepackage{mathtools}
\setcitestyle{numbers,square} 

\newcommand{\RR}{\mathbb{R}}
\newcommand{\EE}{\mathbb{E}}
\newcommand{\ww}{\mathbf{w}}
\newcommand{\pp}{\mathbf{p}}
\newcommand{\cC}{\mathcal{C}}
\newcommand{\cQ}{\mathcal{Q}}
\newcommand{\overbar}[1]{\mkern 1.5mu\overline{\mkern-1.5mu#1\mkern-1.5mu}\mkern 1.5mu}
\newcommand{\bigL}{L^2(\Omega,\sigma; \mathbb{R}^{d})}
\newcommand{\bigLnosigma}{L^2(\Omega; \mathbb{R}^{d})}

\newcommand{\normusigma}{\norm{u}_{H^1(\Omega,\sigma)}}
\newcommand{\onelip}{1\text{-Lip}(\Omega)}
\newcommand{\honesigma}{H^1(\Omega,\sigma)}
\newtheorem{innercustomthm}{Theorem}
\newenvironment{customthm}[1]
  {\renewcommand\theinnercustomthm{#1}\innercustomthm}
  {\endinnercustomthm}

\DeclareMathOperator*{\dist}{dist}
\DeclareMathOperator*{\diam}{diam}

\newcommand\vsni{\vspace{.1 in}\noindent}

\DeclarePairedDelimiter\norm{\lVert}{\rVert}%
\DeclarePairedDelimiter\abs{\lvert}{\rvert}%

\makeatletter
\let\oldabs\abs
\def\abs{\@ifstar{\oldabs}{\oldabs*}}
\let\oldnorm\norm
\def\norm{\@ifstar{\oldnorm}{\oldnorm*}}

\makeatother

\title[WGANs with Gradient Penalty Compute Congested Transport]{Wasserstein GANs with Gradient Penalty Compute Congested Transport}
\usepackage{times}

\coltauthor{%
 \Name{Tristan Milne} \Email{tmilne@math.toronto.edu}\\
 \addr Department of Mathematics, University of Toronto, 40 St. George Street, Toronto Ontario Canada
 \AND
 \Name{Adrian Nachman} \Email{nachman@math.toronto.edu}\\
 \addr Department of Mathematics, University of Toronto, 40 St. George Street, Toronto Ontario Canada\\
 The Edward S. Rogers Sr. Department of Electrical and Computer Engineering, University of Toronto, 10 King's College Road, Toronto Ontario Canada
}

\begin{document}
\maketitle

\begin{abstract}%
Wasserstein GANs with Gradient Penalty (WGAN-GP) are a very popular method for training generative models to produce high quality synthetic data. While WGAN-GP were initially developed to calculate the Wasserstein $1$ distance between generated and real data, recent works (e.g. \cite{stanczuk2021wasserstein}) have provided empirical evidence that this does not occur, and have argued that WGAN-GP perform well not in spite of this issue, but because of it. In this paper we show for the first time that WGAN-GP compute the minimum of a different optimal transport problem, the so-called congested transport \cite{carlier2008optimal}. Congested transport determines the cost of moving one distribution to another under a transport model that penalizes congestion. For WGAN-GP, we find that the congestion penalty has a spatially varying component determined by the sampling strategy used in \cite{gulrajani2017improved} which acts like a local speed limit, making congestion cost less in some regions than others. This aspect of the congested transport problem is new, in that the congestion penalty turns out to be unbounded and depends on the distributions to be transported, and so we provide the necessary mathematical proofs for this setting. One facet of our discovery is a formula connecting the gradient of solutions to the optimization problem in WGAN-GP to the time averaged momentum of the optimal mass flow. This is in contrast to the gradient of Kantorovich potentials for the Wasserstein 1 distance, which is just the normalized direction of flow. Based on this and other considerations, we speculate on how our results explain the observed performance of WGAN-GP. Beyond applications to GANs, our theorems also point to the possibility of approximately solving large scale congested transport problems using neural network techniques.
\end{abstract}

\section{Introduction}
\label{sec:introduction}
Wasserstein GANs (WGANs) were first proposed in \cite{arjovsky2017wasserstein} as a means of training generative models using the Wasserstein 1 distance to measure the dissimilarity of the real and generated distributions. Recall that the Wasserstein 1 distance (or Earth Mover's Distance) between two probability distributions $\mu$ and $\nu$ on a subset $\Omega$ of $\RR^d$ can be calculated via duality as
\begin{equation}
W_1(\mu, \nu) = \sup_{u \in \onelip} \EE_{x \sim \mu}[u(x)] - \EE_{y \sim \nu}[u(y)],\label{prob:W1dual}
\end{equation}
where $\onelip$ is the set of $1$-Lipschitz real valued functions on $\Omega$. The authors of \cite{arjovsky2017wasserstein} showed that $W_1(\mu,\nu)$ has better theoretical properties than the Jenson-Shannon divergence used in the original GAN paper \cite{goodfellow2014generative}, and proposed to solve \eqref{prob:W1dual} by parametrizing $u$ as a neural network $u_w$ with parameters $w$. This network, dubbed the ``critic'', replaced the discriminator of the original GAN. However, designing critics $u_w$ which are $1$-Lipschitz and sufficiently expressive is a non-trivial task which has only more recently seen some progress \cite{anil2019sorting}.

As an initial resolution to this issue, the authors of \cite{arjovsky2017wasserstein} used weight-clipping, but the performance of WGANs in terms of the visual quality of generated images and training stability was greatly improved in \cite{gulrajani2017improved}, wherein weight clipping was discarded in favour of solving the following optimization problem, 
\begin{equation}
\sup_{u} \EE_{x \sim \mu}[u(x)] - \EE_{y \sim \nu}[u(y)] - \lambda \EE_{z \sim \sigma}[(|\nabla u(z)|-1)^2].\label{eq:wgan-gpproblem}
\end{equation}
Here the hard constraint in \eqref{prob:W1dual} that $u$ be $1$-Lipschitz (i.e. $|\nabla u| \leq 1$ almost everywhere) is replaced with a penalty term which penalizes $|\nabla u|$ being different from $1$. The probability distribution $\sigma$ is defined by the following procedure: independently sample $x \sim \mu$, $y \sim \nu$, and $t$ from the uniform distribution on $[0,1]$ to obtain the point $z \sim \sigma$ given by the formula
\begin{equation*}
z = (1-t)x + ty.
\end{equation*}
It was observed in \cite{gulrajani2017improved} that the removal of the hard constraint enforced by weight-clipping allowed for a much more expressive class of functions at the cost of sacrificing guarantees of the functions being $1$-Lipschitz. Generative models trained with this approach are called WGAN-GP (GP for gradient penalty), and they have emerged as a popular method for stably training generative models to produce high-quality synthetic images \cite{karras2018progressive, kurach2019large}. 

The authors of \cite{gulrajani2017improved} also suggested modifying \eqref{eq:wgan-gpproblem} to
\begin{equation}
\sup_{u} \EE_{x \sim \mu}[u(x)] - \EE_{y \sim \nu}[u(y)] - \lambda\EE_{z \sim \sigma}[(|\nabla u(z)|-1)_+^2] \label{eq:onesidedwgan-gpproblem},
\end{equation}
which uses a one-sided gradient penalty ($a_+:= \max(a, 0)$ for $a\in \RR$). This is the version of WGAN-GP studied in this paper. This is because it is a more natural penalty term than the two-sided penalty for encouraging functions to be 1-Lipschitz, since it only penalizes gradients larger than one in norm. Further, it has the important advantage of being convex in $u$. Finally, studies (e.g.  \cite{gulrajani2017improved, petzka2018regularization}) have shown that it obtains equal or better performance than the two-sided penalty in \eqref{eq:wgan-gpproblem}, and it has seen repeated use (see, for example, \cite{lunz2018adversarial, milne2021trust, petzka2018regularization}).

While WGAN-GP have enjoyed spectacular success, the question of whether they are actually computing the Wasserstein 1 distance has only been studied more recently in, for example, \cite{mallasto2019well}, \cite{pinetz2019estimation}, and \cite{stanczuk2021wasserstein}. In particular we were intrigued by \cite{stanczuk2021wasserstein}, which offers empirical evidence that WGAN-GP do \textit{not} compute $W_1(\mu,\nu)$, and, due to some issues with the Wasserstein $1$ distance, argues that this might be the reason for their success. In this paper we analyse this further by establishing for the first time the quantity that WGAN-GP with one-sided penalty \textit{do} compute: it is a congested transport distance. 

We will explain precisely what this means in Section \ref{sec:mainresults}, but vaguely, congested transport is a branch of optimal transport theory that seeks to model the optimal flow of mass under the effects of congestion.   That is, like moving through a busy city, the concentration of mass in a region affects the minimal time required to move through that region. This is distinguished from the typical optimal transport model, where the cost of moving mass from some point $x$ to another $y$ does not take into account the traffic along the way. The theory of congested transport for the continuous case (as opposed to the discrete one) was developed in \cite{brasco2010congested} and \cite{carlier2008optimal}; see also Chapter 4 of \cite{santambrogio2015optimal}.

The contributions of our paper are as follows:
\begin{enumerate}[i.]
\item Under mild assumptions on $\mu$ and $\nu$, we establish that the optimal value \eqref{eq:onesidedwgan-gpproblem} of WGAN-GP is equal to the minimal cost of moving $\mu$ to $\nu$ as determined by a congested transport model. \label{informalclaim:valiscongestedtransport}
\item We establish that the value of this minimal cost is not equal to $W_1(\mu,\nu)$ if $\mu \neq \nu$. More precisely, it is strictly larger for all $\lambda>0$, and at best it converges to $W_1(\mu,\nu)$ like $\lambda^{-1}$ as $\lambda \rightarrow \infty$, if it converges at all. \label{informalclaim:CTcostisnotW1cost}
\item Under slightly stronger assumptions on $\mu$ and $\nu$, we prove the existence of a solution $u_0$ to \eqref{eq:onesidedwgan-gpproblem} in an appropriate function space. \label{informalclaim:existenceofsolutions}
\item\label{informalclaim:gradudeterminesavgmomentum} Further, we show that there is a formula relating $\nabla u_0(x)$ to the time averaged momentum (i.e. mass times velocity) at $x$ of the optimal mass flow for the congested transport problem. This is in contrast to the standard Wasserstein 1 framework, where the gradient of a solution to \eqref{prob:W1dual} is just the direction of optimal mass transport, and does not encode the speed of that transport or the amount of mass transported.
\item Finally, we show that the statement in \ref{informalclaim:gradudeterminesavgmomentum} also holds approximately for approximate solutions of \eqref{eq:onesidedwgan-gpproblem}, which is significant in practice since numerical algorithms for solving \eqref{eq:onesidedwgan-gpproblem} will only produce approximate solutions. \label{informalclaim:approximategradudeterminesavgmomentum}
\end{enumerate} 
The plan for this paper is as follows. In Section \ref{sec:backgroundonCT} we provide the basic background on congested transport needed to state our main theorems. These are given in Section \ref{sec:mainresults}. We discuss related work in Section \ref{sec:relatedwork}. Proof sketches for our main results are presented in Section \ref{sec:proofofmaintheorem}, while the details are deferred to Appendix \ref{sec:appendix}. We provide in Section \ref{sec:discussion} some intuition for how our results may explain the observed performance of WGAN-GP, and summarize the paper in Section \ref{sec:conclusion}.
\section{Background on congested transport}
\label{sec:backgroundonCT}
In this section we will highlight the aspects of congested transport required to present our results. For more background see \cite{brasco2010congested} and \cite{carlier2008optimal}, or Chapter 4 of \cite{santambrogio2015optimal}.

The standard optimal transport problem for probability distributions $\mu$ and $\nu$ on a subset $\Omega$ of Euclidean space $\RR^d$ with cost $c$ is given by
\begin{equation}
\inf\{ \int_{\Omega \times \Omega} c(x,y) d\gamma \mid \gamma \in \mathcal{P}(\Omega \times \Omega), (\pi_x)_\# \gamma = \mu, (\pi_y)_\# \gamma = \nu\}. \label{prob:standardOTproblem}
\end{equation}
Here $\mathcal{P}(\Omega \times \Omega)$ is the set of probability distributions on $\Omega\times \Omega$, so that $\gamma$ can be thought of as a joint probability distribution of random variables $X$ and $Y$ taking values in $\Omega$. The maps $\pi_x, \pi_y : \Omega \times \Omega \rightarrow \Omega$ are the standard projections
\begin{equation*}
\pi_x(x,y) = x, \pi_y(x,y) = y,
\end{equation*}
and the pushforward measures $(\pi_x)_\# \gamma$, $(\pi_y)_\# \gamma$ are the marginals of $\gamma$. As alluded to above, the cost of moving one unit of mass from $x$ to $y$ is $c(x,y)$, which depends only on the initial and final position and not on the path taken between those points, nor on the presence or absence of other mass on that path. Based on this insight, and on earlier work for discrete problems \cite{wardrop1952road}, the theory of congested transport was developed in \cite{carlier2008optimal} to account for possible congestion effects. In this theory, the so-called ``transport plan'' $\gamma \in \mathcal{P}(\Omega\times \Omega)$ of \eqref{prob:standardOTproblem} is replaced with a probability distribution $Q$ (called a ``traffic plan'') on the space of absolutely continuous curves in $\Omega$ parametrized on $[0,1]$, a space we denote by $\mathcal{C}$. It is helpful to introduce for each $t \in [0,1]$ the ``evaluation at time $t$'' map $e_t: \cC \rightarrow \Omega$, which sends a curve $\omega \in \cC$ to its position at time $t$, $\omega(t)$. The curve of measures $(e_t)_\# Q$ is then a flow of mass in time, and compatibility of $Q$ with the source and target distributions $\mu$ and $\nu$ is enforced by requiring
\begin{equation}
(e_0)_\# Q = \mu, (e_1)_\# Q = \nu, \label{eq:compatibilityofQwithmunu}
\end{equation}
which means that the flow $(e_t)_\#Q$ starts at $\mu$ and ends at $\nu$. The set of traffic plans satisfying \eqref{eq:compatibilityofQwithmunu} will be denoted by $\mathcal{Q}(\mu,\nu)$. 

A result in \cite{carlier2008optimal} helps to gain intuition on such traffic plans. Define $\mathcal{C}^{x,y}$ as the subset of $\cC$ consisting of curves that start at $x$ and end at $y$. It is shown in \cite{carlier2008optimal} that for any $Q \in \cQ(\mu,\nu)$ there is a transport plan $\gamma$ admissible in \eqref{prob:standardOTproblem} and for each $x, y \in \Omega$ a distribution $Q^{x,y}$ on $\cC^{x,y}$ such that $Q$ decomposes as $dQ = dQ^{x,y} d\gamma$. More precisely, for any continuous test function $\phi$ mapping $\cC$ to $\RR$,
\begin{equation*}
\int_\cC \phi(\omega) dQ = \int_{\Omega \times \Omega} \int_{\cC^{x,y}} \phi(\omega) dQ^{x,y} d\gamma.
\end{equation*}
In this way, traffic plans $Q$ not only select the initial and final positions of mass with $\gamma$, but also the paths taken between $x$ and $y$ with $Q^{x,y}$. 

To each traffic plan $Q$ one associates a scalar measure $i_Q$, called the traffic intensity, and a vector measure $\ww_Q$, called the traffic flow. For $\phi\in C(\Omega)$ and $\xi \in C(\Omega;\RR^d)$ scalar and vector test functions, respectively, these measures are defined by the equalities
\begin{align}
\int_\Omega \phi di_Q &= \int_\cC \int_0^1 \phi(\omega(t)) |\omega'(t)| dt dQ,\label{def:traffic_intensity}\\
\int_\Omega \xi \cdot d\ww_Q &= \int_\cC \int_0^1 \xi(\omega(t)) \cdot \omega'(t) dt dQ. \label{def:traffic_flow}
\end{align}
Heuristically, for a measurable set $E \subset \Omega$, $i_Q(E)$ represents the total mass passing through $E$ according to $Q$, weighted by the length of each curve in $E$; in this sense $i_Q$ is a measure of congestion. The value $\ww_Q(E)$ also has a physical interpretation. After dividing by the total elapsed time, we can think of it as the time averaged momentum (since it has units of velocity multiplied by mass) of curves, according to $Q$, which pass through $E$. 

We can now define the important notion of the cost of congested transport. To each traffic plan $Q$ that has a traffic intensity $i_Q$ with a density with respect to Lebesgue measure (which we will denote by $i_Q(x)$), one associates a cost given by the formula,
\begin{equation*}
\int_\Omega H(x, i_Q(x)) dx,
\end{equation*}
where $H:\Omega\times \RR \rightarrow \RR$ is the cost function. In general, $H$ has a non-trivial dependency on the spatial variable $x$ (and in our case this is important) but the example often studied (e.g. in \cite{brasco2010congested}) is
\begin{equation}
H(x, z) = H(z) = \frac{1}{2\lambda}z^2 + |z|. \label{eq:exampleH}
\end{equation}
With this choice of cost, the incremental cost at congestion level $i_Q(x)$ is $H'(i_Q(x)) = \frac{1}{\lambda}i_Q(x) + 1$. As such, when congestion (i.e. $i_Q(x)$) is large, the incremental cost of adding more mass at $x$ is very high. Conversely, when there is no traffic (i.e. $i_Q(x) = 0$), there is still a non-zero incremental cost; this is often phrased as ``cars cannot travel at infinite speeds on empty roads''. 

The standard congested transport problem is then to minimize the cost given by $H$ among all traffic plans $Q\in \cQ(\mu,\nu)$ with $i_Q$ absolutely continuous with respect to Lebesgue measure; that is
\begin{equation*}
\inf \{ \int_\Omega H(x, i_Q(x)) dx \mid Q \in \cQ(\mu,\nu), i_Q \ll \mathcal{L}_d\}.
\end{equation*}
See \cite{carlier2008optimal} for proofs of existence of solutions to this problem, and \cite{brasco2010congested} for the relationship between solutions and minimal flows.
\section{Main results}
\label{sec:mainresults}
\subsection{Definition of the optimization problems}
Here we will define more precisely the optimization problems involved in WGAN-GP. Throughout we will assume that $\mu$ and $\nu$ have densities $f(x)$ and $g(x)$  with respect to Lebesgue measure; we will see later in Section \ref{sec:propertiesofsigma} that this implies $\sigma$ also has a density, which we will denote by $\sigma(x)$. 

We need to specify the space over which we are maximizing in \eqref{eq:onesidedwgan-gpproblem}. The simplest choice is $H^1(\Omega)$, the Sobolev space of functions $u:\Omega \rightarrow \RR$ with $u$ and its weak derivative $\nabla u$ satisfying
\begin{equation*}
\int_\Omega (u^2(x) + |\nabla u|^2(x)) dx < \infty.
\end{equation*} 
With this set of admissible functions, the problem, which we denote as \eqref{prob:GPlambda}, is as follows:
\begin{equation}
\sup\{ \langle u,f-g \rangle - \frac{\lambda}{2}\int_\Omega(|\nabla u| - 1)_+^2 \sigma(x) dx \mid u \in H^1(\Omega)\}\tag{$GP_\lambda$},\label{prob:GPlambda}
\end{equation}
where $\langle u, f-g \rangle$ denotes the $L^2(\Omega)$ inner product. The functional in \eqref{prob:GPlambda} is exactly that of \eqref{eq:onesidedwgan-gpproblem} up to a rescaling of $\lambda$. We emphasize that in \eqref{prob:GPlambda}, as in later problems, we work specifically with the measure $\sigma$ that corresponds precisely to the sampling scheme in \cite{gulrajani2017improved}. The value of the supremum will be denoted as $\sup\eqref{prob:GPlambda}$, and throughout this paper we will use max (or min, as appropriate) rather than sup (or inf) when the optimization problem has a solution. We use the space $H^1(\Omega)$ as it is the simplest space over which we can guarantee that the functional is finite if, say, $f, g \in L^2(\Omega)$ and $\sigma \in L^\infty(\Omega)$. 

A slightly more complicated but also more natural space, in view of the penalty term, is the weighted Sobolev space $H^1(\Omega,\sigma)$. This is the space of functions $u$ with weak derivatives $\nabla u$ having finite norm according to the weight $\sigma$, i.e.
\begin{equation*}
\normusigma := \left(\int_\Omega (u^2(x) + |\nabla u|^2(x))\sigma(x) dx\right)^{1/2} < \infty.
\end{equation*}
As we will see later we have sharper results for the case of optimizing \eqref{eq:onesidedwgan-gpproblem} over this space. \textit{A priori} it is not clear that the measure $\sigma$ from \cite{gulrajani2017improved} is non-degenerate almost everywhere in $\Omega$ and so it may not serve as a reasonable weight, but in Section \ref{sec:propertiesofsigma} we will provide conditions on $f$ and $g$ which guarantee this. Proceeding under the assumption that $\sigma$ is a reasonable weight, we write our second problem, denoted \eqref{prob:tildeGPlambda} as
\begin{equation}
\sup\{ \langle u, f-g \rangle - \frac{\lambda}{2}\int_\Omega(|\nabla u| - 1)_+^2 \sigma(x) dx \mid u \in H^1(\Omega, \sigma)\}.\tag{$\widetilde{GP}_\lambda$}\label{prob:tildeGPlambda}
\end{equation}
We will show in Section \ref{sec:proofofstrongtheorem} that $\langle u, f-g\rangle$ has meaning for $u \in H^1(\Omega,\sigma)$, even though $u$ may not be in $L^2(\Omega)$. This will rely on the dependence of the sampling measure $\sigma$ from \cite{gulrajani2017improved} on the given distributions $\mu$ and $\nu$. Aside from this interesting feature of the problem, the only difference between \eqref{prob:GPlambda} and \eqref{prob:tildeGPlambda} is the space over which we optimize; $H^1(\Omega, \sigma)$ is a natural choice given the structure of the gradient penalty term. However, for $\sigma$ to be a reasonable weight we require some additional assumptions on $f$ and $g$, so we include results for \eqref{prob:GPlambda}, valid even when these assumptions fail. We will denote the value of $\eqref{prob:tildeGPlambda}$ as $\sup\eqref{prob:tildeGPlambda}$. For background on weighted Sobolev spaces, see for instance \cite{kufner1985weighted} or \cite{kufner1984define}.

We are now ready to state our congested transport problem, denoted \eqref{prob:CPlambda}:
\begin{equation}
\inf \{ \int_\Omega H(x, i_Q(x)) dx \mid Q \in \mathcal{Q}(\mu,\nu), i_Q \ll \mathcal{L}_d, i_Q \in L^2(\Omega)\},\tag{$CP_\lambda$}\label{prob:CPlambda}
\end{equation}
where $H: \Omega \times \RR \rightarrow \RR$ is given by
\begin{equation}
H(x, z) = \begin{cases}
\frac{1}{2\lambda \sigma(x)}z^2 + |z| &\quad \sigma(x) >0,\\
0 &\quad \sigma(x) = 0, z = 0,\\
+\infty &\quad \sigma(x) = 0, z \neq 0.
\end{cases}\label{eq:Hdeff}
\end{equation}
Remarkably, the cost $H$ that appears in our congested transport problem is very close to the one given in \eqref{eq:exampleH}, which is the prototypical example in the congested transport literature. The only difference is that our $H$ contains a spatially varying component depending on $\sigma$, which in turn depends on $\mu$ and $\nu$. One can think of $\sigma$ as a local speed limit, in the sense that where $\sigma$ is large the cost for a certain amount of mass flow is less, and vice versa. 
\subsection{Statement of the main results}
Throughout we will assume that $\Omega \subset \RR^d$ is an open, bounded, convex set with a Lipschitz boundary. The following theorem is a formal statement of our contributions \ref{informalclaim:valiscongestedtransport}, \ref{informalclaim:CTcostisnotW1cost}, and \ref{informalclaim:approximategradudeterminesavgmomentum} as listed in Section \ref{sec:introduction}.
\begin{customthm}{\textbf{A}}
\label{thm:weaktheorem}
Suppose that $f$ and  $g$ are probability density functions in $L^\infty(\Omega)$. Then
\begin{enumerate}
\item $\sup \eqref{prob:GPlambda} = \inf \eqref{prob:CPlambda}< +\infty$ (i.e. WGAN-GP compute a congested transport cost). \label{claim:valiscongestedtransport}
\item There exists a $C>0$, such that for all $\lambda >0$,
\label{claim:strictlymorethanW1}
\begin{equation*}
\sup \eqref{prob:GPlambda} \geq W_1(\mu,\nu)\left(1 + \frac{C}{\lambda}W_1(\mu,\nu)\right).
\end{equation*}
In particular, $\sup \eqref{prob:GPlambda} > W_1(\mu,\nu)$ for all $\lambda>0$ whenever $W_1(\mu,\nu)$ is non-zero.
\item There is a traffic plan $Q_0$ that solves \eqref{prob:CPlambda}. Moreover, any two solutions $Q_0, Q_1$ have the same traffic flow and traffic intensity, which are related by the equations \label{claim:existenceofQsolution}
\begin{equation*}
\ww_{Q_0} = \ww_{Q_1}, i_{Q_0} = |\ww_{Q_0}| = |\ww_{Q_1}| = i_{Q_1},
\end{equation*}
where $|\ww_{Q_0}|$ and $|\ww_{Q_1}|$ are the total variation measures of the vector measures $\ww_{Q_0}$ and $\ww_{Q_1}$.
\item If $Q_0$ is a solution to \eqref{prob:CPlambda} and if $u_0\in H^1(\Omega)$ is an approximate solution to \eqref{prob:GPlambda} in the sense that for some $\epsilon>0$,
\begin{equation*}
\langle u_0, f-g \rangle - \frac{\lambda}{2}\int_\Omega (|\nabla u_0| -1)_+^2 \sigma dx \geq  \sup \eqref{prob:GPlambda} - \epsilon,
\end{equation*}
then the vector density $\ww_{Q_0}$ for the traffic flow satisfies, for some constant $C$,\label{claim:approximateextremality}
\begin{equation}
\norm{\ww_{Q_0} + \lambda \sigma (|\nabla u_0|-1)_+ \frac{\nabla u_0}{|\nabla u_0|}}^2_{L^2(\Omega;\RR^d)} \leq C \epsilon.\label{eq:approximateextremality}
\end{equation}
\end{enumerate}
\end{customthm}

With some simple additional assumptions on $f$ and $g$, we are able to obtain the existence of a solution $u_0 \in H^1(\Omega,\sigma)$ to \eqref{prob:tildeGPlambda} and make \eqref{eq:approximateextremality} an identity. Theorem \ref{thm:strongtheorem} is a formal statement of these results, and addresses contributions \ref{informalclaim:existenceofsolutions} and \ref{informalclaim:gradudeterminesavgmomentum} from Section \ref{sec:introduction}. 
\begin{customthm}{\textbf{B}}
\label{thm:strongtheorem}
Assume that $f$ and $g$ are probability density functions in $L^\infty(\Omega)$, that $g = 0$ in a neighbourhood of the boundary of $\Omega$, and  that $\inf_{x \in \Omega} f(x) >0$. Then
\begin{enumerate}
\item All of the claims of Theorem \ref{thm:weaktheorem} hold,\label{claim:allofAholds}
\item \eqref{prob:tildeGPlambda} has a solution, \label{claim:tildegplambdahassolution}
\item $\sup \eqref{prob:GPlambda} = \max  \eqref{prob:tildeGPlambda}$, and\label{claim:equivalenceofsupandmax}
\item If $u_0$ solves \eqref{prob:tildeGPlambda} and $Q_0$ solves \eqref{prob:CPlambda}, then \label{claim:strongextremality}
\begin{align}
\mathbf{w}_{Q_0} = -\lambda \sigma (|\nabla u_0| -1)_+ \frac{\nabla u_0}{|\nabla u_0|}, \quad \nabla u_0(x) = -\left(\frac{1}{\lambda \sigma(x)} + \frac{1}{|\ww_{Q_0}(x)|}\right)\ww_{Q_0}(x),\label{eq:strongextremality}
\end{align}
the latter formula holding for almost all $x$ such that $\ww_{Q_0}(x) \neq 0$.
\end{enumerate}
\end{customthm}
\section{Related Work}
Since our work connects two previously disconnected fields (i.e. congested transport and generative modelling with WGANs), we will briefly review related works from both areas here.

To our knowledge, ours are the first results in the WGAN literature determining precisely the optimal value of the objective function from \cite{gulrajani2017improved} with one-sided penalty. More generally, several papers (e.g. \cite{mallasto2019well, pinetz2019estimation, stanczuk2021wasserstein}) have provided empirical evidence that WGAN-GP do not compute $W_1(\mu,\nu)$, but do not offer a precise notion of what they do compute. In particular, \cite{mallasto2019well} shows that for simple discrete problems WGAN-GP tend to over-estimate the Wasserstein $1$ distance. A similar observation is noted in \cite{pinetz2019estimation}, who also provide evidence that the optimal value computed by WGAN-GP tends to converge to $W_1(\mu,\nu)$ only as $\lambda \rightarrow \infty$. Statement \ref{claim:strictlymorethanW1} of Theorem \ref{thm:weaktheorem} provides a theoretical explanation for both of these observations.  

The subject of one versus two-sided penalties for WGAN-GP is studied in depth in \cite{petzka2018regularization}, which shows empirically that the former results in more stable training of WGANs with less dependence on the $\lambda$ parameter. Despite this, it seems that the two-sided penalty is the default for many practitioners (e.g. \cite{kurach2019large}). Our results only apply directly to the one-sided penalty, but they are useful for bounding the optimal value of the two-sided penalty using the bounds provided in \cite{petzka2018regularization}; a more precise analysis for the two-sided penalty may be obtainable through a convexification argument.

An interesting recent work is \cite{biau2021some}, which deals with some of the issues that arise when approximating a Kantorovich potential with neural networks, though does not consider the functional used in WGAN-GP. Since our work leaves such approximation issues unexamined, we view \cite{biau2021some} as an important and complementary perspective.

On the congested transport side, we note that our proof of the equivalence of $\sup\eqref{prob:GPlambda}$ and $\inf\eqref{prob:CPlambda}$ is inspired by \cite{brasco2010congested} and refinements in \cite{brasco2014continuous}. Indeed, arguments in \cite{brasco2010congested} show this equivalence for the case of $\sigma(x) =1$ almost everywhere. However, we were unable to find the results we needed for the $\sigma$ from \cite{gulrajani2017improved} in the literature; \cite{brasco2010congested} works with the case of $\sigma = 1$, and \cite{brasco2014continuous} includes generalizations of these results to the case of inhomogeneous (i.e. $x$ dependent) cost $H(x,z)$. However, they assume $H$ is bounded above for fixed $z$, which is not the case for WGAN-GP (see \eqref{eq:Hdeff}). Thus one can view our results as extending some of the work in both \cite{brasco2010congested} and \cite{brasco2014continuous} to a new class of problems where the cost function $H(x,z)$ appearing in the congested transport problem is unbounded in $x$; we obtain this extension by exploiting the dependence of $\sigma$ on the distributions $\mu$ and $\nu$.

Let us also note that there are variants of congested transport which are not isotropic, in the sense that the cost function depends on the direction of traffic flow (e.g. \cite{brasco2013congested}). We speculate that this might lead to interesting generalizations of WGAN-GP, but leave this for future work.

\label{sec:relatedwork}

\section{Proof of the main theorems}
\label{sec:proofofmaintheorem}
In this section we will sketch the proofs of Theorems \ref{thm:weaktheorem} and \ref{thm:strongtheorem}; detailed proofs are provided in Appendix \ref{sec:appendix}. The proof of Theorem \ref{thm:weaktheorem} is analogous to the approach from congested transport theory (i.e. \cite{brasco2014continuous} and \cite{brasco2010congested}). We consider a dual problem given by the following Beckmann type problem:

\begin{equation}
\inf\{\int_\Omega H(x, |\ww|(x)) dx \mid \mathbf{w} \in \bigLnosigma, \nabla \cdot \mathbf{w} = f-g \},\tag{$BP_\lambda$}
\label{prob:BPlambda}
\end{equation}
where $H$ is given in \eqref{eq:Hdeff}, and as is usual for Beckmann problems  the equation $\nabla \cdot \ww = f-g$ holds in the weak sense. In other words, $\ww$ is admissible in \eqref{prob:BPlambda} only if, for all $u \in H^1(\Omega)$, 
\begin{equation*}
    \int_{\Omega} -\nabla u \cdot \ww dx = \int_{\Omega} u (f-g ) dx.
\end{equation*}
Note that since we do not impose conditions on the behaviour of $u$ at the boundary of $\Omega$, this imposes a no-flux condition on $\ww$. We then prove that $\sup\eqref{prob:GPlambda}$ equals $\inf \eqref{prob:BPlambda}$, and that the latter is equal to $\inf\eqref{prob:CPlambda}$. The relationship \eqref{eq:approximateextremality} between approximate solutions of \eqref{prob:GPlambda} and \eqref{prob:CPlambda} is then established using properties of the Legendre dual of the function $H(x,z)$ defined in \eqref{eq:Hdeff}.

To prove Theorem \ref{thm:strongtheorem}, we show in Proposition \ref{prop:comparabletodistance3} that $\sigma(x)$ is comparable to the distance to the boundary function
\begin{equation*}
\dist(x,\partial \Omega) := \inf_{y \in \partial \Omega} |x-y|.
\end{equation*} 
This establishes some properties of $H^1(\Omega,\sigma)$ which we need for our proof, such as a Poincar{\'e} inequality and density of smooth functions up to the boundary. The former is used to establish the existence of a minimizer to \eqref{prob:tildeGPlambda}, and the latter is used to establish $\sup \eqref{prob:GPlambda} = \max \eqref{prob:tildeGPlambda}$. Finally, the equalities in \eqref{eq:strongextremality} are obtained, essentially, by evaluating \eqref{eq:approximateextremality} at a solution to \eqref{prob:tildeGPlambda}.

\subsection{Properties of $\sigma$}
\label{sec:propertiesofsigma}
We begin by collecting several properties of the measure $\sigma$ which will be needed in the proofs of Theorems \ref{thm:weaktheorem} and \ref{thm:strongtheorem}. If $x$, $y$, and $t$ are sampled independently (as is done in \cite{gulrajani2017improved}), $\sigma$ is the probability distribution defined by the formula
\begin{equation*}
\sigma = \pi_\# (U[0,1] \otimes \mu \otimes \nu),
\end{equation*}
where $U[0,1]$ is the uniform measure on $[0,1]$, $U[0,1] \otimes \mu \otimes \nu$ is the product measure, and $\pi: [0,1] \times \Omega \times \Omega \rightarrow \Omega$ is the map
\begin{equation*}
\pi(t, x, y) = (1-t) x + ty.
\end{equation*}
The following lemma guarantees that if $\mu$ and $\nu$ have densities with respect to Lebesgue measure on $\RR^d$ then $\sigma$ also has a density. We also provide a formula for this density. 
\begin{lemma}
\label{lem:sigmaACandformula}
If $\mu = f(x) dx$ and $\nu = g(x) dx$ then $\sigma $ has a density given by the formula
\begin{equation}
\sigma(z) = \int_0^1 \int_\Omega f\left(\frac{z-ty}{1-t}\right)g(y) (1-t)^{-d} dy dt.\label{eq:sigmaformula}
\end{equation}
\end{lemma}
The proof is a straightforward calculation; see Appendix \ref{sec:proofsofsigmaproperties} for details.

We next establish that $\sigma$ is bounded when $f$ and $g$ are bounded. In particular this shows that the functional in \eqref{prob:GPlambda} is finite over $H^1(\Omega)$ for bounded $f$ and $g$. See Appendix \ref{sec:proofsofsigmaproperties} for a full proof.
\begin{lemma}
If $f$ and $g$ are probability density functions in $L^\infty(\Omega)$, then so is $\sigma$. 
\label{lem:sigmaLinfinity}
\end{lemma}
Finally we establish that under the assumptions of Theorem \ref{thm:strongtheorem}, $\sigma$ is comparable to $\dist(x, \partial \Omega)$.
\begin{proposition}
\label{prop:comparabletodistance3}
Assume that $f$ and $g$ are probability density functions in $L^\infty(\Omega)$, that $g = 0$ in a neighbourhood of the boundary of $\Omega$, and  that $\inf_{x \in \Omega} f(x) >0$. Then there exists a constant $C$ such that for all $x \in \Omega$,
\begin{equation*}
\frac{1}{C}\dist(x, \partial \Omega) \leq \sigma(x) \leq C \dist(x, \partial \Omega).
\end{equation*}
\end{proposition}
Proposition \ref{prop:comparabletodistance3} is proved in Appendix \ref{sec:proofsofsigmaproperties} in Lemmas \ref{lem:distancelowerbound}, \ref{lem:y-xissptaty} and \ref{lem:distupperbound}. The lower bound on $f$ is sufficient for the lower bound on $\sigma$, while convexity of $\Omega$ and having $g$ vanish in a  neighbourhood of $\partial \Omega$ is sufficient for the upper bound. Note that in this case $\sigma$ vanishes on $\partial \Omega$, and hence the cost $H(x,z)$ is unbounded as $x$ approaches the boundary.

\subsection{Proof of Theorem \ref{thm:weaktheorem}}
\label{sec:proofofweaktheorem}
For convenience, we denote the functional arising in \eqref{prob:GPlambda} as
 $J :H^1(\Omega) \times \bigLnosigma \rightarrow \RR$,
\begin{align}
J(u, \pp) = \langle u, g-f \rangle + \frac{\lambda}{2}\int_\Omega (|\pp|-1)_+^2 \sigma(x) dx. \label{eq:Jdeff}
\end{align}
We note that since $\sigma\in L^\infty(\Omega)$ (see Lemma \ref{lem:sigmaLinfinity}), $J$ does indeed take finite values.  The problem \eqref{prob:GPlambda} can then be written as
\begin{equation*}
-\inf_{u \in H^1(\Omega)} J(u,\nabla u).
\end{equation*}
The proof of Theorem \ref{thm:weaktheorem} rests on the following inequality, which takes into account the dependence of $\sigma$ on the given probability density functions $f$ and $g$. A full proof is given in Appendix \ref{sec:proofsofresultsforweaktheorem}.
\begin{lemma}
\label{lem:supGPlambdaisfinite}
If $f$ and $g$ are probability density functions in $L^\infty(\Omega)$, and $\sigma$ is the density given in \eqref{eq:sigmaformula}, then for all $u \in H^1(\Omega)$ we have
\begin{equation}
|\langle u, g-f \rangle | \leq \diam(\Omega) \left(\int_\Omega |\nabla u|^2 \sigma(x) dx \right)^{1/2},\label{eq:boundingugminusf}
\end{equation}
where $\diam(\Omega) = \sup\{ |x-y| \mid x, y \in \Omega\}$. Further, $\sup \eqref{prob:GPlambda}$ is finite.
\end{lemma}

\subsubsection{Equality of $\sup\eqref{prob:GPlambda}$ and $\inf\eqref{prob:BPlambda}$}
Here we address the equality of $\sup\eqref{prob:GPlambda}$ and $\inf \eqref{prob:BPlambda}$. The proof also establishes the existence of a unique solution to \eqref{prob:BPlambda}.
\begin{proposition}
\label{prop:weakbetweenGPandBP}
If $f$ and $g$ are probability density functions in $L^\infty(\Omega)$, then
\begin{equation}
\sup \eqref{prob:GPlambda} = \inf \eqref{prob:BPlambda}.\label{eq:weakGP=BP}
\end{equation}
Furthermore, \eqref{prob:BPlambda} has a unique minimizer. 
\end{proposition}
The proof consists of verifying the conditions needed for strong duality in convex analysis (see, for example, \cite{ekeland1999convex}). To do so, we rely on Lemma \ref{lem:supGPlambdaisfinite}, as well as convexity and continuity of $J$. See Appendix \ref{sec:proofsofresultsforweaktheorem} for the details. 

\subsubsection{Equivalence of \eqref{prob:BPlambda} and \eqref{prob:CPlambda}}
The following result, with Proposition \ref{prop:weakbetweenGPandBP}, will complete the proof of statements \ref{claim:valiscongestedtransport} and \ref{claim:existenceofQsolution} of Theorem \ref{thm:weaktheorem}. After using the existence of a solution to \eqref{prob:BPlambda}, proved in Proposition \ref{prop:weakbetweenGPandBP}, the proof for unbounded $H(x,z)$ is essentially the same as in \cite{brasco2014continuous} and relies on the monotonicity of $H(x,z)$ in $|z|$. 
\begin{proposition}
\label{prop:weakbetweenBPandCP}
If $f$ and $g$ are probability density functions in $L^\infty(\Omega)$, we have
\begin{equation}
\min \eqref{prob:BPlambda} = \inf \eqref{prob:CPlambda}. \label{eq:equalityBPCP}
\end{equation}
Moreover, \eqref{prob:CPlambda} has a solution, and any solution $Q_0$ to \eqref{prob:CPlambda} is related to the unique solution $\mathbf{w}_0$ to \eqref{prob:BPlambda} through the equations
\begin{equation}
\mathbf{w}_{Q_0} = \mathbf{w}_0, \quad i_{Q_0} = |\mathbf{w}_0|. \label{eq:averagevelocityissolutiontobeckmann}
\end{equation}
\end{proposition}
\begin{proof}{\textbf{sketch}}
We first show that for each $Q$ admissible in \eqref{prob:CPlambda}, the traffic flow $\mathbf{w}_Q$ is admissible for \eqref{prob:BPlambda} and satisfies $i_Q \geq |\mathbf{w}_Q|$; this, together with monotonicity of $H(x,z)$ in $z$, will establish
\begin{equation}
\inf (CP_\lambda) \geq \inf(BP_\lambda).\label{eq:CPlambdanolessthanBPlambda}
\end{equation}
To prove the reverse, we show that for every $\ww$ admissible in \eqref{prob:BPlambda} there is a traffic plan $Q$ with $i_Q \leq |\ww|$ via Theorem 4.10 of \cite{santambrogio2015optimal}. The proof of existence of a solution $Q_0$ and its relation to $\ww_0$ then follows from this argument together with Proposition \ref{prop:weakbetweenGPandBP}. See Appendix \ref{sec:proofsofresultsforweaktheorem} for a details.
\end{proof}
Having verified statements \ref{claim:valiscongestedtransport} and \ref{claim:existenceofQsolution} of Theorem \ref{thm:weaktheorem}, we now turn to statements \ref{claim:strictlymorethanW1} and \ref{claim:approximateextremality}.
\subsubsection{Inequality between $\sup \eqref{prob:GPlambda}$ and $W_1(\mu,\nu)$}

Here we will establish, using Proposition \ref{prop:weakbetweenGPandBP}, that the value computed by solving \eqref{prob:GPlambda} is strictly larger than $W_1(\mu,\nu)$ for all $\lambda>0$ if $\mu \neq \nu$, and that at best $\sup \eqref{prob:GPlambda}$ decays to $W_1(\mu,\nu)$ like $\lambda^{-1}$ as $\lambda\rightarrow + \infty$. This is statement \ref{claim:strictlymorethanW1} of Theorem \ref{thm:weaktheorem}.
\begin{lemma}
\label{lem:wassersteinonestrictlylessthanCT}
If $f$ and $g$ are probability density functions in $L^\infty(\Omega)$, there exists a constant $C>0$ depending on $\norm{\sigma}_{L^\infty(\Omega)}$ such that
\begin{equation*}
\sup \eqref{prob:GPlambda} \geq  W_1(\mu,\nu)\left(1 + \frac{C}{\lambda}W_1(\mu,\nu)\right).
\end{equation*} 
\end{lemma}
\begin{proof}\textbf{sketch}
The proof follows by bounding
\begin{equation*}
\inf \eqref{prob:BPlambda} \geq \inf_{\nabla\cdot \ww = \mu-\nu} \int_\Omega\frac{1}{2\lambda \sigma}|\ww|^2 dx + \inf_{\nabla\cdot \ww = \mu-\nu} \int_\Omega |\ww|(x) dx,
\end{equation*}
and estimating each term individually; the second of these is known to be $W_1(\mu,\nu)$ (see, for example, \cite{santambrogio2015optimal} Chapter 4). See Appendix \ref{sec:proofsofresultsforweaktheorem} for full details.
\end{proof} 
\subsubsection{Relationship between approximate solutions of \eqref{prob:GPlambda} and solutions of \eqref{prob:CPlambda}}
Statement \ref{claim:approximateextremality} of Theorem \ref{thm:weaktheorem} is an immediate consequence of the following proposition.
\begin{proposition}
\label{prop:globalestimate}
If $f$ and $g$ are probability density functions in $L^\infty(\Omega)$ and $\ww_0$ is a solution to \eqref{prob:BPlambda}, then for all $u \in H^1(\Omega)$,
\begin{equation}
\sup\eqref{prob:GPlambda} + J(u, \nabla u)  \geq \frac{1}{2\lambda \norm{\sigma}_{L^\infty(\Omega)}}\norm{\ww_0 + \lambda \sigma (|\nabla u|-1)_+ \frac{\nabla u}{|\nabla u|}}^2_{L^2(\Omega)}.
\label{eq:globalestimate}
\end{equation}
\end{proposition}
See Appendix \ref{sec:proofsofresultsforweaktheorem} for a detailed proof, which comes from estimates on $H$ and its Legendre dual.
\subsection{Proof of Theorem \ref{thm:strongtheorem}}
\label{sec:proofofstrongtheorem}
We begin by recording some properties of $H^1(\Omega,\sigma)$, necessary for the proof of Theorem \ref{thm:strongtheorem}, that come from Proposition \ref{prop:comparabletodistance3}.
\subsubsection{Properties of $H^1(\Omega,\sigma)$}
\begin{lemma}
Under the assumptions of Theorem \ref{thm:strongtheorem}, $H^1(\Omega,\sigma)$ has a Poincar{\'e} inequality. That is, for all $u \in H^1(\Omega,\sigma)$,
\begin{equation}
\int_\Omega (u(x) - (u)_\sigma)^2 \sigma(x)dx \leq C \int_\Omega |\nabla u(x)|^2\sigma(x) dx,\label{eq:Poincareinequality}
\end{equation}
where
\begin{equation*}
(u)_\sigma = \int_\Omega u(x) \sigma(x) dx.
\end{equation*}
Further, $C^\infty( \overbar{\Omega} )$ is dense in $H^1(\Omega,\sigma)$ in the $H^1(\Omega,\sigma)$ norm.\label{lem:propertiesofH1sigma}
\end{lemma}
Lemma \ref{lem:propertiesofH1sigma} follows by noting that the desired properties hold for the weighted Sobolev space with weight $\rho(x) =  d(x,\partial \Omega)$, and using Proposition \ref{prop:comparabletodistance3} to show that they are then inherited by $H^1(\Omega,\sigma)$. See Appendix \ref{sec:proofsofresultsforstrongtheorem} (in particular, Lemmas \ref{lem:Poincareinequality} and \ref{lem:densityofsmoothfunctions}) for detailed proofs.

Note that since inequality \eqref{eq:boundingugminusf} holds for all $u \in C^\infty(\overbar{\Omega})$, density of this space in $H^1(\Omega,\sigma)$ makes $\langle u, g-f \rangle$ well defined for all $u \in H^1(\Omega,\sigma)$. 
\subsubsection{Proof sketch for Theorem \ref{thm:strongtheorem}}
We now leverage the properties of $\honesigma$ to prove Theorem \ref{thm:strongtheorem}. Lemma \ref{lem:tildeJcoerciveandinGamma0} below summarizes the required properties of the corresponding functional.

\vsni
Observing that the functional in \eqref{prob:tildeGPlambda} is invariant under the map $u \mapsto u + c$, we may restrict without loss of generality to the Hilbert space
\begin{equation*}
\bar{H}^1(\Omega,\sigma) = \{ u \in H^1(\Omega,\sigma) \mid \int_\Omega u(x) \sigma(x) dx = 0\},
\end{equation*}
where the norm is the same as the one for $H^1(\Omega,\sigma)$. Define $\tilde{J}:\bar{H}^1(\Omega,\sigma) \rightarrow \RR$ as
\begin{equation}
\tilde{J}(u) = \langle u, g-f \rangle + \frac{\lambda}{2}\int_\Omega (|\nabla u|-1)_+^2 \sigma(x) dx.\label{eq:tildeJdeff}
\end{equation}
It is clear that $\sup \eqref{prob:tildeGPlambda} = - \displaystyle{\inf_{u \in \bar{H}^1(\Omega,\sigma)} \tilde{J}(u,\nabla u)}$.
\begin{lemma}
Under the assumptions of Theorem \ref{thm:strongtheorem}, $\tilde{J}$ is coercive on $\bar{H}^1(\Omega,\sigma)$ (i.e. $\tilde{J}(u) \rightarrow +\infty$ if $\normusigma \rightarrow +\infty$). Moreover, $\tilde{J}$ is convex, proper, and continuous, and hence also weakly lower semi-continuous.\label{lem:tildeJcoerciveandinGamma0}
\end{lemma}
This is proved in Appendix \ref{sec:proofsofresultsforstrongtheorem}, and provides the results needed for the proof of Theorem \ref{thm:strongtheorem}.

\vsni

\begin{proof}{\textbf{sketch for Theorem \ref{thm:strongtheorem}}}
Let us note that the hypotheses of Theorem \ref{thm:weaktheorem} are subsumed by those of Theorem \ref{thm:strongtheorem}, so we need only prove statements \ref{claim:tildegplambdahassolution} through \ref{claim:strongextremality}. The proof of statement \ref{claim:tildegplambdahassolution} of Theorem \ref{thm:strongtheorem} follows from Lemma \ref{lem:tildeJcoerciveandinGamma0} and the direct method of the calculus of variations. The equivalence of the optimal values in statement \ref{claim:equivalenceofsupandmax} follows since the functionals in \eqref{prob:GPlambda} and \eqref{prob:tildeGPlambda} agree on $C^\infty(\overbar{\Omega})$ and are continuous on both $H^1(\Omega)$ and $\honesigma$, respectively. Finally, statement \ref{claim:strongextremality} of Theorem \ref{thm:strongtheorem} follows in a similar manner to Proposition \ref{prop:globalestimate}, except that the existence of a solution to \eqref{prob:tildeGPlambda} allows us to make the inequality an equality; inverting the map $\xi \mapsto (|\xi|-1)_+ \xi/|\xi|$ for $|\xi| >1$ then provides the second equality of \eqref{eq:strongextremality}. See Appendix \ref{sec:proofsofresultsforstrongtheorem} for a detailed proof.
\end{proof} 

\section{Discussion and Conclusion}
\label{sec:d_and_c}
\subsection{Discussion}
\label{sec:discussion}
In this section we provide some intuition on how Theorems \ref{thm:weaktheorem} and \ref{thm:strongtheorem} might explain the performance of WGAN-GP. We also discuss how our work could be used to create new methods for approximately solving large scale congested transport problems. 

We first observe that in congested transport concentration of data is penalized, thus encouraging diverse generated distributions. This may explain how WGAN-GP avoid the strong averaging effect of the Wasserstein 1 distance which was observed experimentally in \cite{stanczuk2021wasserstein}. This effect was elucidated in \cite{stanczuk2021wasserstein} by an inequality which shows that the average Wasserstein 1 distance between minibatches of a given dataset is larger than the average distance between a minibatch and the global mean of the dataset. We expect that under congested transport this inequality would no longer hold, as the transport to the global mean would be expensive since all mass converges to a single point, thus inducing congestion. 

Second, our formula \eqref{eq:strongextremality} may explain why generators trained with WGAN-GP converge faster than those trained with weight clipping, as observed in \cite{gulrajani2017improved}. In the case of weight clipping, the gradient of the critic is bounded by a fixed constant. On the other hand, formula \eqref{eq:strongextremality} shows that the gradient of the optimal critic for WGAN-GP depends not only on the direction of optimal flow, but also on the momentum via its dependency on $\ww_{Q_0}$ (see \eqref{def:traffic_flow}). Thus, data points with a long distance to travel under the optimal traffic flow would have a high speed under gradient descent on $u_0$. We expect that this would lead to faster convergence of generators trained with WGAN-GP. 

We emphasize that the intuition given here is based on our results but is not rigorously proven. However, it points to the possibility for future analysis of WGAN-GP based on congested transport theory.

We leave unexamined some important problems in this area, including how our results interact with issues that arise in actual implementations. These include the constrained class of functions over which WGAN-GP are optimized (an issue considered in \cite{biau2021some}), issues with approximation of the integrals in \eqref{eq:onesidedwgan-gpproblem} by sampling, and the fact that \eqref{eq:onesidedwgan-gpproblem} is not solved to completion before updating the generated distribution $\mu$. We also anticipate that there will be practical applications of Theorem \ref{thm:weaktheorem} to the training of WGANs, and we plan to address these in a follow up paper.

Beyond applications to WGANs, we note that the documentation for \cite{flamary2021pot} states that existing efficient algorithms for computing the Wasserstein 1 distance are not suitable for large datasets, and that critics trained with WGAN-GP are one recommended approach for treating such problems. Theorem \ref{thm:weaktheorem}, presented here, quantifies the relationship between the congested transport cost computed by WGAN-GP and the Wasserstein 1 distance. As well, our discovery points to the possibility of numerically solving congested transport problems by suitable variants of WGAN-GP.

\subsection{Conclusion}
\label{sec:conclusion}
We have shown, under weak assumptions on $\mu$ and $\nu$, that the value of the optimization problem for WGAN-GP with one-sided penalty is equal to the minimal cost of moving $\mu$ to $\nu$ under a congested transport model. This cost was shown to be strictly larger than $W_1(\mu,\nu)$ for all $\lambda>0$ provided $\mu \neq \nu$, with an error that is at least $O(\lambda^{-1})$ as $\lambda \rightarrow \infty$. We have also proved, under slightly stronger assumptions on $\mu$ and $\nu$, that there are solutions to this problem in an appropriate weighted Sobolev space. The gradients of these solutions are shown to encode more information on the optimal transport of mass than do the gradients of a standard Kantorovich potential (i.e. time averaged momentum, as opposed to just direction of transport). We have also shown that an approximate version of this relationship holds for approximate solutions. This work provides theoretical foundations for understanding the performance of WGAN-GP.

\acks{This research was supported in part by the NSERC Discovery Grant RGPIN-06329 and a University of Toronto Doctoral Completion Award.}

\bibliography{C:/Users/Tristan/OneDrive/Documents/U_of_T/Research/AnnotatedBibliography}
\bibliographystyle{alpha}

\appendix
\section{Proofs}
\label{sec:appendix}
In this appendix we will provide detailed proofs of all mathematical results presented in the main body of the paper. The proofs here are organized by the section in which the corresponding result appeared.
\subsection{Proofs of results from Section \ref{sec:propertiesofsigma}}
\label{sec:proofsofsigmaproperties}
\begin{proof}{\textbf{of Lemma \ref{lem:sigmaACandformula}}}
The measure $\sigma$ from \cite{gulrajani2017improved} assigns to a measurable set $E$ the value
\begin{align*}
\sigma(E) &= \int_0^1 \int_{\RR^d} \int_{\RR^d} 1_E((1-t) x + ty) f(x) g(y) dx dy dt.
\end{align*}
Here we have extended the densities $f$ and $g$ by zero to $\RR^d$. For $y$ and $t$ fixed with $t \neq 1$, we set $z = (1-t)x + ty$. Applying this change of variable to the integral with respect to $x$, we obtain
\begin{align*}
\sigma(E) &= \int_0^1 \int_{\RR^d} \int_{\RR^d} 1_E(z) f\left(\frac{z-ty}{1-t}\right) g(y) (1-t)^{-d} dz dy dt, \\
&= \int_E \int_0^1 \int_{\Omega} f\left(\frac{z-ty}{1-t}\right) g(y) (1-t)^{-d}  dy dt dz,
\end{align*}
the second line being obtained by an application of Fubini's theorem and recalling that $g$ is supported in $\Omega$. This proves \eqref{eq:sigmaformula}.
\end{proof}

\begin{proof}{\textbf{of Lemma \ref{lem:sigmaLinfinity}}}
We write
\begin{equation*}
\sigma(z) = \int_0^{1/2} \int_\Omega f\left(\frac{z-ty}{1-t}\right) g(y) (1-t)^{-d}  dy dt + \int_{1/2}^1 \int_\Omega f\left(\frac{z-ty}{1-t}\right) g(y) (1-t)^{-d}  dy dt.
\end{equation*}
For fixed $z$ and $t \neq 1$, apply the change of variable $y' = \frac{z-ty}{1-t}$ to the second integral. This gives
\begin{equation*}
\int_{1/2}^1 \int_\Omega f\left(\frac{z-ty}{1-t}\right) g(y) (1-t)^{-d}  dy dt = \int_{1/2}^1 \int_\Omega f(y') g\left(\frac{z-(1-t)y'}{t}\right)t^{-d} dy' dt.
\end{equation*}
As such,
\begin{align*}
\sigma(z) &\leq \int_0^{1/2} \int_\Omega \norm{f}_{L^\infty(\Omega)} g(y) (1-t)^{-d} dy dt + \int_{1/2}^1 \int_\Omega \norm{g}_{L^\infty(\Omega)} f(y') t^{-d} dy' dt,\\
&= C(\norm{f}_{L^\infty(\Omega)} + \norm{g}_{L^\infty(\Omega)}),
\end{align*}
which proves the claim. 
\end{proof}

Lemmas \ref{lem:distancelowerbound} and \ref{lem:distupperbound} below prove Proposition \ref{prop:comparabletodistance3}.
\begin{lemma}
Suppose that $\inf_{x\in \Omega} f(x)= \epsilon >0$. Then there exists $C$, depending only on $\Omega$, such that
\begin{equation}
\sigma(x) \geq C \epsilon \dist(x, \partial \Omega). \label{eq:distancelowerbound}
\end{equation} 
\label{lem:distancelowerbound}
\end{lemma}
\begin{proof}
Let $x \in \Omega$ with $\dist(x,\partial \Omega) >0$. Then for $t \in [0,1)$ and $y \in \Omega$,
\begin{equation*}
|x - \frac{x-ty}{1-t}| = \frac{t}{1-t}|x-y| \leq \frac{t}{1-t}\diam(\Omega),
\end{equation*}
and $\frac{t}{1-t} \diam(\Omega) \leq \dist(x,\partial \Omega)$ if and only if
\begin{equation*}
t \leq \frac{\dist(x, \partial \Omega)}{\diam(\Omega) + \dist(x, \partial \Omega)}=:t_0.
\end{equation*}
As such, for all $t \in [0, t_0]$ and $y \in \Omega$, $\frac{x-ty}{1-t} \in \Omega$, and hence the lower bound on $f$ applies. Thus,
\begin{align*}
\sigma(x) &= \int_0^1 \int_\Omega f\left(\frac{x-ty}{1-t}\right) g(y) (1-t)^{-d} dy dt,\\
&\geq \int_0^{t_0}\int_\Omega \epsilon g(y) (1-t)^{-d} dy dt,\\
&= \epsilon \int_0^{t_0} (1-t)^{-d}dt,\\
&\geq \epsilon t_0,
\end{align*}
where the last line follows from noting that $(1-t)^{-d} \geq 1$ for all $t \in [0, 1)$. Hence
\begin{align*}
\sigma(x) &\geq \epsilon t_0,\\
&= \epsilon \frac{\dist(x, \partial \Omega)}{\diam(\Omega) + \dist(x, \partial \Omega)},\\
&\geq C\epsilon \dist(x, \partial \Omega),
\end{align*}
for
\begin{equation*}
C = \frac{1}{\diam(\Omega) + \sup_{x \in \Omega} \dist(x, \partial \Omega)}.
\end{equation*}
\end{proof}
To establish an upper bound on $\sigma$ in terms of $\dist(x, \partial \Omega)$ we need a basic lemma about convex sets. This result is standard, but we include it here for the convenience of the reader. 
\begin{lemma}
If $\Omega$ is convex, $x \in \Omega$ and $y \in \partial \Omega$ satisfying
\begin{equation*}
|x-y| = \dist(x, \partial \Omega),
\end{equation*}
then
\begin{equation}
\langle y - x, y \rangle = \sup_{z \in \Omega} \langle y -x ,z \rangle.\label{eq:y-xissptaty}
\end{equation}
\label{lem:y-xissptaty}
\end{lemma}
\begin{proof}
Since $y \in \partial \Omega$ we immediately have
\begin{equation*}
\langle y-x, y \rangle \leq \sup_{z \in \Omega} \langle y-x, z \rangle.
\end{equation*}
Now suppose that $y$ does not obtain the supremum in \eqref{eq:y-xissptaty}. Then  there exists $z \in \Omega$ such that
\begin{equation*}
\langle y- x, y \rangle < \langle y-x ,z \rangle.
\end{equation*} 
Observe that for all $t >1$, the point $(1-t)z + ty \not \in \Omega^o$ (here $\Omega^o$ is the interior of $\Omega$), since otherwise we would have $y \in \Omega^o$ by virtue of $\Omega$ being convex.

\vsni
Next we assert that for all $t>1$ and small enough, the point $(1-t)z + ty$ is strictly closer to $x$ than $y$ is. Indeed,
\begin{align*}
|(1-t)z + ty - x|^2 &= |(1-t)z - (1-t)y - (x-y)|^2,\\
&= (1-t)^2|z-y|^2 +2(1-t) \langle z-y, y-x \rangle + |x-y|^2,\\
&= (t-1)((t-1)|z-y|^2 - 2 \langle z-y, y-x\rangle) + |x-y|^2.
\end{align*} 
Since $\langle z-y, y-x \rangle >0$, the assertion is proven. For $t>1$ and small enough there must therefore exist a point on the segment connecting $x$ and $(1-t)z + ty$ which is in $\partial \Omega$ (since $(1-t)z + ty \not \in \Omega^o$) and is closer to $x$ than $y$, contradicting the assumption that $y$ is the closest point to $x$ in $\partial \Omega$. 
\end{proof}
We can now establish the desired upper bound provided $\Omega$ is convex and $g$ is zero in a neighbourhood of the boundary of $\Omega$.
\begin{lemma}
\label{lem:distupperbound}
If $\Omega$ is convex and open, $f, g \in L^\infty(\Omega)$, and $g = 0$ in a neighbourhood of the boundary of $\Omega$, then there exists a constant $C$ such that for all $x \in \Omega$,
\begin{equation*}
\sigma(x) \leq C \dist(x, \partial \Omega).
\end{equation*}
\end{lemma}
The idea of the proof is to partition $\Omega$ into two sets by a certain half-plane, one set in which $g$ is zero, and another set in which the time integral in \eqref{eq:sigmaformula} can be restricted to a sub-interval $[0, t_0]$ where $t_0$ is related to $\dist(x, \partial \Omega)$. 
\begin{proof}
Let $x \in \Omega$, and let $y \in \partial \Omega$ such that
\begin{equation*}
|x-y| = \dist(x, \partial \Omega).
\end{equation*}
Set $\theta = :\langle y - x, y \rangle$ and note that by Lemma \ref{lem:y-xissptaty}, the condition $\langle y-x, z \rangle >\theta$ implies that $z \in \Omega^c$. For $\delta>0$ to be determined, set
\begin{equation*}
H^-(x) = \{ z \in \Omega \mid \langle y- x, z \rangle \leq \theta - \delta \}, H^+(x) = \{ z \in \Omega \mid \langle y- x, z \rangle > \theta - \delta \}.
\end{equation*}
Since $H^+(x)$ and $H^-(x)$ partition $\Omega$, we can write
\begin{equation*}
\sigma(x) = \int_{H^-(x)} \int_0^1 f\left(\frac{x-tz}{1-t}\right)g(z)(1-t)^{-d} dt dz + \int_{H^+(x)} \int_0^1 f\left(\frac{x-tz}{1-t}\right)g(z)(1-t)^{-d} dt dz.
\end{equation*}
Since $g = 0$ in a neighbourhood of $\partial \Omega$, there exists a constant $\delta_0$ such that if $g(z) >0$, then $\dist(z, \partial \Omega) > \delta_0$. We assert that the condition $\delta \leq \delta_0 \dist(x, \partial \Omega)$ then implies that $g(z) = 0$ on $H^+(x)$. Indeed, if $z \in H^+(x)$ then
\begin{equation*}
\langle z + \frac{\delta}{\dist(x, \partial \Omega)^2}(y-x), y-x \rangle > \theta-\delta +\delta = \theta. 
\end{equation*}
Hence $z + \frac{\delta}{\dist(x, \partial \Omega)^2}(y-x) \in \Omega^c$, indicating that
\begin{equation*}
\dist(z, \partial \Omega) \leq \frac{\delta}{\dist(x, \partial \Omega)} \leq \delta_0,
\end{equation*}
and hence $g(z) = 0$, proving the assertion. Therefore, if $\delta \leq \delta_0 \dist(x, \partial \Omega)$, we have
\begin{equation*}
\sigma(x) = \int_{H^-(x)} \int_0^1 f\left(\frac{x-tz}{1-t}\right)g(z)(1-t)^{-d} dt dz.
\end{equation*}
We next determine, given $z \in H^-(x)$, for what range of $t \in [0,1]$ we have
\begin{equation*}
\frac{x-tz}{1-t} \in \Omega.
\end{equation*}
Again we use Lemma \ref{lem:y-xissptaty}. For $z \in H^-(x)$ we compute
\begin{align*}
\langle \frac{x-tz}{1-t}, y-x \rangle &= \frac{1}{1-t}\langle x, y-x \rangle - \frac{t}{1-t} \langle z, y-x \rangle,\\
&\geq \frac{1}{1-t}(\theta - \dist(x, \partial \Omega)^2) - \frac{t}{1-t}(\theta - \delta),\\
&= \theta + \frac{1}{1-t}(t\delta - \dist(x, \partial \Omega)^2).
\end{align*}
Thus, if $t\in (\frac{\dist(x, \partial \Omega)^2}{\delta},1]$ and $z \in H^-(x)$, we have $\frac{x-tz}{1-t} \in \Omega^c$. 

\vsni
Under the assumption that $\dist(x, \partial \Omega) \leq \frac{\delta_0}{2}$, we now select
\begin{equation*}
\delta := \delta_0 \dist(x, \partial \Omega).
\end{equation*} 
With this choice, the interval $(\frac{\dist(x, \partial \Omega)^2}{\delta}, 1]$ is non-empty since $\frac{\dist(x, \partial \Omega)^2}{\delta} = \frac{\dist(x, \partial \Omega)}{\delta_0} \leq \frac{1}{2}$. Since $\delta \leq \delta_0 \dist(x, \partial \Omega)$, we use our work above to conclude that
\begin{align*}
\sigma(x) &= \int_{H^{-}(x)} \int_0^{\frac{\dist(x, \partial \Omega)}{\delta_0}} f \left(\frac{x-tz}{1-t}\right) g(z) (1-t)^{-d} dt dz,\\
&\leq \norm{f}_{L^\infty(\Omega)}\int_{H^{-}(x)} \int_0^{\frac{\dist(x, \partial \Omega)}{\delta_0}} (1-t)^{-d} g(z) dt dz.
\end{align*}
Noting that $(1-t)^{-d}$ is an increasing function on $[0,1)$ and $\frac{\dist(x, \partial \Omega)}{\delta_0} \leq \frac{1}{2}$, we get
\begin{align*}
\sigma(x) &\leq  \norm{f}_{L^\infty(\Omega)} \frac{2^d}{\delta_0}\dist(x, \partial \Omega).
\end{align*}
For $x$ with $\dist(x, \partial \Omega) > \frac{\delta_0}{2}$ we use Lemma \ref{lem:sigmaLinfinity} to conclude
\begin{align*}
\sigma(x) &\leq \norm{\sigma}_{L^\infty(\Omega)},\\
&\leq \frac{2\norm{\sigma}_{L^\infty(\Omega)}}{\delta_0} \dist(x, \partial \Omega).
\end{align*}
Hence, for all $x \in \Omega$,
\begin{equation*}
\sigma(x) \leq \max\left( \frac{2}{\delta_0}\norm{\sigma}_{L^\infty(\Omega)},  \frac{2^d}{\delta_0}\norm{f}_{L^\infty(\Omega)}\right) \dist(x, \partial \Omega).
\end{equation*}
\end{proof}
\subsection{Proofs of results from Section \ref{sec:proofofweaktheorem}}
\label{sec:proofsofresultsforweaktheorem}
\begin{proof}{\textbf{of Lemma \ref{lem:supGPlambdaisfinite}}}
For $u \in C^\infty(\overbar{\Omega})$, we have
\begin{align*}
|\langle u, g-f \rangle|&= |\int_\Omega u(x)(g(x) - f(x)) dx|,\\
&= |\int_\Omega \int_\Omega \int_0^1 \nabla u((1-t) x + ty) \cdot (y-x) dt f(x) g(y) dx dy|,\\
&\leq \diam(\Omega) \int_\Omega \int_\Omega \int_0^1 |\nabla u((1-t)x + ty)| f(x) g(y) dt dx dy,\\
&= \diam(\Omega) \int_\Omega |\nabla u(x)| \sigma(x)dx.
\end{align*}
Thus, via Cauchy-Schwarz, 
\begin{equation*}
|\langle u, g-f \rangle| \leq \diam(\Omega) \left(\int_\Omega |\nabla u|^2 \sigma(x) dx \right)^{1/2}.
\end{equation*}
By density of $C^\infty(\overbar{\Omega})$ in $H^1(\Omega)$, \eqref{eq:boundingugminusf} therefore holds for all $u \in H^1(\Omega)$. As such, for $u \in H^1(\Omega)$,
\begin{align*}
J(u, \nabla u) &\geq - \diam(\Omega) \left(\int_\Omega |\nabla u|^2 \sigma(x) dx \right)^{1/2} + \frac{\lambda}{2}\int_\Omega (|\nabla u| -1)_+^2 \sigma(x) dx,\\
&\geq - \diam (\Omega) \left(\int_\Omega (|\nabla u|-1)_+^2 \sigma(x) dx \right)^{1/2} + \frac{\lambda}{2}\int_\Omega (|\nabla u| -1)_+^2 \sigma(x) dx - C,\\
&= \frac{\lambda}{2}\left( \left(\int_\Omega (|\nabla u| -1)_+^2 \sigma(x) dx\right)^{1/2} - \frac{\diam(\Omega)}{\lambda}\right)^2 - C,\\
&> -C,
\end{align*}
where the second inequality is obtained by applying the triangle inequality twice. This proves that $\displaystyle{\inf_{u \in H^1(\Omega)} J(u, \nabla u)}$ (and hence $\sup \eqref{prob:GPlambda}$) is finite.   
\end{proof}

To established the duality result in Proposition \ref{prop:weakbetweenGPandBP}, we need some simple properties of the functional $J$, which we prove here.
\begin{lemma}
\label{lem:propertiesofweakJ}
The functional $J$ defined in \eqref{eq:Jdeff} is convex and continuous.
\end{lemma}
\begin{proof}
Define $F: H^1(\Omega) \rightarrow \RR$ and $G : \bigLnosigma \rightarrow \RR$ as
\begin{equation*}
F(u) = \langle u, g-f \rangle, \quad G(\pp) = \frac{\lambda}{2}\int_\Omega (|\pp|-1)_+^2 \sigma(x) dx,
\end{equation*}
so that $J(u, \pp) = F(u) + G(\pp)$.
To see that $J$ is convex and continuous, we claim that $F$ and $G$ are both convex and continuous on their respective domains. For $F$ these claims are immediate since $F$ is a continuous linear operator. For $G$ they require proof. Note that $G$ can be written
\begin{equation*}
G(\mathbf{p}) = \int_\Omega h(x, \pp(x)) dx,
\end{equation*}
where $h: \Omega \times \RR^d \rightarrow \RR$ is given by
\begin{equation}
h(x, \xi) = \frac{\lambda}{2}(|\xi|-1)_+^2 \sigma(x).\label{eq:hdeff}
\end{equation}
Note that for fixed $x$, $h$ is convex in $\xi$ because it is the composition of a convex function $(\xi \mapsto |\xi|$) with a non-decreasing convex function ($z \mapsto \frac{\lambda}{2}(z-1)_+^2\sigma(x)$). Convexity of $h$ then implies convexity of $G$. 

\vsni
To show continuity of $G$, observe that for $\mathbf{p}_1, \mathbf{p}_2 \in L^2(\Omega;\RR^d)$,
\begin{align*}
G(\mathbf{p}_1) - G(\mathbf{p}_2) &= \frac{\lambda}{2}\int_\Omega (|\mathbf{p}_1| -1)_+ + (|\mathbf{p}_2| -1)_+)(|\mathbf{p}_1| -1)_+ - (|\mathbf{p}_2| -1)_+) \sigma(x) dx,\\
&\leq \frac{\lambda}{2}\norm{\sigma}_{L^\infty(\Omega)}(\norm{\mathbf{p}_1}_{\bigLnosigma} + \norm{\mathbf{p}_2}_{\bigLnosigma}) \\
&\quad \times\left(\int_\Omega\left((|\mathbf{p}_1|-1)_+ - (|\mathbf{p}_2|-1)_+\right)^2 dx \right)^{1/2},\\
&\leq \frac{\lambda}{2}\norm{\sigma}_{L^\infty(\Omega)}(\norm{\mathbf{p}_1}_{\bigLnosigma} + \norm{\mathbf{p}_2}_{\bigLnosigma}) \norm{\mathbf{p}_1-\mathbf{p}_2}_{\bigLnosigma},
\end{align*}
where in the last line we have used the fact that $\xi \mapsto (|\xi|-1)_+$ is $1$-Lipschitz. Swapping the roles of $\pp_1$ and $\pp_2$, we obtain that $G$ is continuous.
\end{proof}
The same proof, with minor modifications, establishes the following related result.
\begin{lemma}
The map $\tilde{G}:\bigL\rightarrow \RR$ given by
\begin{equation}
\tilde{G}(\mathbf{p}) = \frac{\lambda}{2}\int_\Omega(|\mathbf{p}|-1)_+^2 \sigma(x) dx\label{eq:tildeGdef}
\end{equation}
is continuous and convex. \label{lem:Gonweightedspaceconvexandcontinuous}
\end{lemma}
\begin{proof}{\textbf{of Proposition \ref{prop:weakbetweenGPandBP}}}
We apply Theorem 4.1 of Chapter 3 of \cite{ekeland1999convex}, which states that if
\begin{enumerate}[i.]
\item $J$ is convex,
\item $\inf_{u \in H^1(\Omega)} J(u, \nabla u)$ is finite, and
\item there exists $u_0 \in H^1(\Omega)$ such that $J(u_0, \nabla u_0) < \infty$ with the function $\pp \mapsto J(u_0, \pp)$ being continuous at $\nabla u_0$,
\end{enumerate}
then with $J^*$ as the Legendre dual of $J$,
\begin{equation*}
\sup_{\ww \in \bigLnosigma} - J^*(\nabla^* \ww, -\ww) = \inf_{u \in H^1(\Omega)} J(u, \nabla u),
\end{equation*}
and the problem on the left hand side has at least one solution. By Lemma \ref{lem:dualcalc}, this is equivalent to \eqref{eq:weakGP=BP} and existence of a solution to \eqref{prob:BPlambda}. Note that this solution must be unique by strict convexity of the functional in $\eqref{prob:BPlambda}$. Thus if we can verify points i - iii we are done, and these are shown in Lemmas \ref{lem:supGPlambdaisfinite} and \ref{lem:propertiesofweakJ}.
\end{proof}

\begin{lemma}
Identifying $\bigLnosigma$ as its own dual space, we have that 
\begin{equation*}
J^*(\nabla^* \ww, -\ww) = \begin{cases}
\int_\Omega H(x, |\ww|(x)) dx &\quad \text{ if } \quad \nabla \cdot \ww = f-g,\\
+\infty &\quad \text{ else}.
\end{cases}
\end{equation*}
\label{lem:dualcalc}
\end{lemma}
\begin{proof}{\textbf{of Lemma \ref{lem:dualcalc}}}
We have
\begin{equation*}
J^*(u^*, \mathbf{p}^*) = F^*(u^*) + G^*(\mathbf{p}^*).
\end{equation*}
Since $F$ is linear, the definition of the Legendre dual gives that
\begin{equation*}
F^*(u^*) = \begin{cases}
0 &\quad  u^* = \langle \cdot, g-f \rangle,\\
+\infty &\quad \text{ else,}
\end{cases} = 1_{g-f}(u^*).
\end{equation*}
To calculate $G^*(\mathbf{p}^*)$, we use Proposition 2.1 from Chapter 9 of \cite{ekeland1999convex}, which relies on the measurable selection theorem. In this context we write
\begin{equation*}
G(\mathbf{p}) = \int_\Omega h(x, \mathbf{p}(x)) dx,
\end{equation*}
where $h$ is as in \eqref{eq:hdeff}. We verify the hypotheses of this proposition, which are
\begin{enumerate}[i.]
\item that $\Omega$ is a bounded open subset of $\RR^d$
\item that $h$ is a non-negative normal integrand (see Definition 1.1, Chapter 8, \cite{ekeland1999convex}), and
\item that there exists $\mathbf{p} \in L^\infty(\Omega;\RR^d)$ with $G(\mathbf{p}) < + \infty$.
\end{enumerate}
The third point is clear for $\mathbf{p} = 0$. The second point follows because $h$ is a Carath{\'e}odory function, which are proven to be normal integrands in Proposition 1.1, Chapter 8, \cite{ekeland1999convex}; recall that a function $h: \Omega \times \RR^d \rightarrow \RR$ is said to be a Carath{\'e}odory function if 
\begin{itemize}
\item for almost all $x \in \Omega$, $h(x, \cdot)$ is continuous on $\RR^d$, and
\item for almost all $\xi \in \RR^d$, $h(\cdot, \xi)$ is measurable on $\Omega$.
\end{itemize}
Both of these clearly hold for $h$ given in \eqref{eq:hdeff}. The conclusion of Proposition 2.1 from Chapter 9 of \cite{ekeland1999convex} is that 
\begin{equation*}
G^*(\mathbf{p}^*) = \int_\Omega h^*(x, \mathbf{p}^*(x)) dx, 
\end{equation*}
for $h^*$ the partial dual
\begin{equation*}
h^*(x, \xi^*) = \sup_{\xi \in \RR^d} \xi^* \cdot \xi - h(x, \xi).
\end{equation*}
An elementary calculation gives that
\begin{equation*}
h^*(x, \xi^*) = H(x, |\xi^*|),
\end{equation*}
where $H$ is as in \eqref{eq:Hdeff}. Relabelling the variable as $\mathbf{w}$ for consistency, we have
\begin{equation*}
J^*(\nabla^* \mathbf{w}, -\mathbf{w}) =  1_{g-f}(\nabla^* \mathbf{w}) + \int_\Omega H(x, |\ww|(x)) dx. 
\end{equation*}
An infinite value is obtained for the first term unless $\ww$ is such that
\begin{equation*}
\int_\Omega \ww \cdot \nabla u dx = \int_\Omega u (g-f) dx
\end{equation*}
for all $u \in H^1(\Omega)$. Writing this requirement as $\nabla \cdot \ww = f-g$, the proof is complete.
\end{proof}

\begin{proof}{\textbf{of Proposition \ref{prop:weakbetweenBPandCP}}}
Here we will follow the strategy outlined in the sketch of the proof in the main text. Let $Q$ be admissible for \eqref{prob:CPlambda}. It is known (see \cite{santambrogio2015optimal}, Section 4.2.3) that
\begin{equation}
|\mathbf{w}_Q| \leq i_Q, \label{eq:averagevelocitylessthanaveragespeed}
\end{equation}
where $|\mathbf{w}_Q|$ is the total variation measure of $\mathbf{w}_Q$. Then by \eqref{eq:averagevelocitylessthanaveragespeed} and the definition of \eqref{prob:CPlambda} we know that $|\mathbf{w}_Q| \ll \mathcal{L}_d$ and $\mathbf{w}_Q \in \bigLnosigma$. Take $u \in C^\infty(\overbar{\Omega})$, and observe that
\begin{align*}
\int_\Omega \nabla u\cdot\mathbf{w}_Q dx &= \int_\mathcal{C}\int_0^1 \nabla u(\omega(t)) \cdot \omega'(t) dt dQ,\\
&= \int_\mathcal{C}(u(\omega(1)) - u(\omega(0))) dQ,\\
&= \int_\Omega u (g -f) dx, 
\end{align*}
where in the last line we have used the definition of $\mathcal{Q}(\mu,\nu)$. Thus, by density of $C^\infty(\overbar{\Omega})$ in $H^1(\Omega)$, we have that for all $u\in H^1(\Omega)$,
\begin{equation*}
\int_\Omega\nabla u\cdot\mathbf{w}_Q dx = \langle u, g-f \rangle. 
\end{equation*}
Hence $\mathbf{w}_Q$ is admissible in \eqref{prob:BPlambda}. Moreover, \eqref{eq:averagevelocitylessthanaveragespeed} together with the monotonicity of the integrand in \eqref{prob:CPlambda} gives
\begin{equation*}
\int_\Omega H(x, i_Q(x)) dx \geq \int_\Omega H(x, |\ww_Q|(x)) dx.
\end{equation*}
This establishes $\inf \eqref{prob:CPlambda} \geq \inf \eqref{prob:BPlambda}$. Now let $\mathbf{w}_0$ be a solution to \eqref{prob:BPlambda}. Since $C^\infty(\overbar{\Omega}) \subset H^1(\Omega)$, we have
\begin{equation*}
\int_\Omega \nabla u \cdot \mathbf{w}_0 dx = \int_\Omega u (g-f) dx
\end{equation*}
for all $u \in C^\infty(\overbar{\Omega})$, which is the hypothesis for Theorem 4.10 of \cite{santambrogio2015optimal}. As such, there exists $Q_0 \in \mathcal{P}(\mathcal{C})$ such that $(e_0)_\# Q_0 = \mu, (e_1)_\# Q_0 = \nu$, and
\begin{equation}
i_{Q_0} = |\mathbf{w}_{Q_0}| \leq |\mathbf{w}_0|.\label{eq:averagespeedequaltoaveragevelocity}
\end{equation}
Hence $i_{Q_0} \ll \mathcal{L}_d$ and $i_{Q_0} \in L^2(\Omega)$, so $Q_0$ is admissible for \eqref{prob:CPlambda}. Further, \eqref{eq:averagespeedequaltoaveragevelocity} implies that
\begin{equation*}
\int_\Omega H(x, i_{Q_0}(x)) dx \leq \int_\Omega H(x, |\ww_0|(x)) dx, 
\end{equation*}
and thus $\inf \eqref{prob:CPlambda} \leq \inf \eqref{prob:BPlambda}$, establishing \eqref{eq:equalityBPCP} and that $Q_0$ is a solution to \eqref{prob:CPlambda}.

To prove the final claim, let $Q_0$ be optimal in \eqref{prob:CPlambda}. Then the inequality \eqref{eq:averagevelocitylessthanaveragespeed}, together with \eqref{eq:equalityBPCP}, shows that $\mathbf{w}_{Q_0}$ is optimal for $(BP_\lambda)$. Recalling from Proposition \ref{prop:weakbetweenGPandBP} that $(BP_\lambda)$ has a unique minimizer, we get the first equality of \eqref{eq:averagevelocityissolutiontobeckmann}. To get the second equality we note that it is implied by optimality of $\ww_{Q_0}$ together with \eqref{eq:equalityBPCP} and \eqref{eq:averagevelocitylessthanaveragespeed}.
\end{proof}
\begin{proof}{\textbf{of Lemma \ref{lem:wassersteinonestrictlylessthanCT}}}
By Proposition \ref{prop:weakbetweenGPandBP} and since $\sigma \in L^\infty(\Omega)$, it is clear that
\begin{align}
\sup \eqref{prob:GPlambda} &\geq \inf \{ \int_\Omega \frac{1}{2\lambda\norm{\sigma}_{L^\infty(\Omega)}}|\ww|^2(x) dx \mid \ww \in L^2(\Omega; \RR^d), \nabla \cdot \ww = \mu - \nu\}\nonumber \\
&\quad + \inf\{ \int_\Omega |\ww|(x) dx \mid \ww \in L^2(\Omega; \RR^d), \nabla \cdot \ww = \mu - \nu\}. \label{prob:standardBeckmannwithL2constraint}
\end{align}
Let us analyse the second problem. Following Chapter 4 of \cite{santambrogio2015optimal}, write $\mathcal{M}_\text{div}^d(\Omega)$ as the set of vector measures with divergence which is a scalar measure. It is clear that
\begin{equation*}
\inf\{ \int_\Omega |\ww|(x) dx \mid \ww \in L^2(\Omega; \RR^d), \nabla \cdot \ww = \mu - \nu\} \geq \inf\{ |\ww|(\Omega) \mid \ww \in \mathcal{M}_\text{div}^d(\Omega), \nabla \cdot \ww = \mu - \nu\},
\end{equation*}
where $|\ww|$ is the total variation measure of $\ww$, and it is well known (see Chapter 4 of \cite{santambrogio2015optimal} again) that the right hand side is $W_1(\mu,\nu)$. We use Cauchy-Schwarz in \eqref{prob:standardBeckmannwithL2constraint} to get
\begin{equation*}
\sup \eqref{prob:GPlambda} \geq \frac{1}{2\lambda\norm{\sigma}_{L^\infty(\Omega)} \text{Vol}(\Omega)}  W_1(\mu,\nu)^2 + W_1(\mu,\nu),
\end{equation*}
whence the conclusion follows immediately.
\end{proof}
Lemma \ref{lem:funccoercedbygrad} provides a lower bound on a function $\varphi$ that is used to prove Proposition \ref{prop:globalestimate}.
\begin{lemma}
For $\xi^*\in \RR^d$, define $\varphi: \Omega\times \RR^d \rightarrow \RR \cup \{+\infty\}$ by
\begin{equation*}
\varphi(x, \xi)= \xi^* \cdot \xi + \frac{\lambda \sigma(x)}{2}(|\xi| -1)_+^2 +H(x, |\xi^*|).
\end{equation*}
For all $(x,\xi) \in \Omega \times \RR^d$ we have
\begin{equation}
\varphi(x, \xi) \geq  \frac{1}{2\lambda \norm{\sigma}_{L^\infty(\Omega)}}|\xi^* + \lambda \sigma(x) (|\xi|-1)_+\frac{\xi}{|\xi|}|^2.\label{eq:funccoercedbygrad}
\end{equation}\label{lem:funccoercedbygrad}
\end{lemma}
\begin{proof}
Note that if $\sigma(x) = 0$ and $\xi^* \neq 0$, $\varphi(x, \xi) = +\infty$, and so \eqref{eq:funccoercedbygrad} holds. Further, if $\sigma(x) = 0$ and $\xi^* = 0$, $\varphi(x, \xi) = 0$, and \eqref{eq:funccoercedbygrad} holds. We therefore proceed assuming that $\sigma(x) \neq 0$. Suppose $|\xi| \geq 1$. Through elementary manipulations one can show that
\begin{equation*}
\varphi(x, \xi) = \frac{\lambda\sigma(x)}{2}|\xi - \frac{\xi}{|\xi|} + \frac{\xi^*}{\lambda \sigma(x)}|^2 + \xi^* \cdot \frac{\xi}{|\xi|} + |\xi^*|.
\end{equation*}
Applying the Cauchy-Schwarz inequality,
\begin{align*}
\varphi(x, \xi) &\geq \frac{1}{2\lambda\sigma(x)}|\xi^* + \lambda \sigma(x)(|\xi|-1)_+\frac{\xi}{|\xi|}|^2.
\end{align*}
For $|\xi| \leq 1$, 
\begin{align*}
\varphi(x, \xi) &= \xi^* \cdot \xi + \frac{1}{2\lambda \sigma(x)}|\xi^*|^2 + |\xi^*|,\\
&\geq \frac{1}{2\lambda \sigma(x)}|\xi^*|^2,\\
&= \frac{1}{2\lambda\sigma(x)}|\xi^* + \lambda \sigma(x)(|\xi|-1)_+\frac{\xi}{|\xi|}|^2.
\end{align*}
Noting that
\begin{equation*}
\frac{1}{\sigma(x)} \geq \frac{1}{\norm{\sigma}_{L^\infty(\Omega)}},
\end{equation*}
we obtain the inequality \eqref{eq:funccoercedbygrad}.
\end{proof}

\begin{proof}{\textbf{of Proposition \ref{prop:globalestimate}}}
For $\xi^* = \ww_0(x)$, $\xi = \nabla u(x)$, the inequality \eqref{eq:funccoercedbygrad} gives
\begin{align*}
\nabla u(x) \cdot \ww_0(x) + &\frac{\lambda}{2}(|\nabla u|(x) -1)_+^2 \sigma(x) + H(x, |\ww_0|(x)) \\ &\geq \frac{1}{2\lambda \norm{\sigma}_{L^\infty(\Omega)}}|\ww_0(x) + \lambda \sigma(x) (|\nabla u(x)|-1)_+ \frac{\nabla u(x)}{|\nabla u(x)|}|^2.
\end{align*}
Integrating this over $\Omega$ and using the equality \eqref{eq:weakGP=BP} together with $\nabla \cdot \ww_0 = f-g$, we get
\begin{equation*}
J(u, \nabla u) + \sup\eqref{prob:GPlambda} \geq \int_\Omega \frac{1}{2\lambda \norm{\sigma}_{L^\infty(\Omega)}}|\ww_0(x) + \lambda \sigma(x) (|\nabla u(x)|-1)_+ \frac{\nabla u(x)}{|\nabla u(x)|}|^2 dx,
\end{equation*}
which is \eqref{eq:globalestimate}.
\end{proof}
\subsection{\textbf{Proofs of results from Section \ref{sec:proofofstrongtheorem}}}
\label{sec:proofsofresultsforstrongtheorem}
We start by showing that $\sigma$ being comparable to $\dist(x, \partial \Omega)$ implies that the weighted Sobolev space $H^1(\Omega, \sigma)$ has a Poincar{\'e} inequality (Lemma \ref{lem:Poincareinequality}) and that $C^\infty(\overbar{\Omega})$ is dense in $H^1(\Omega,\sigma)$ (Lemma \ref{lem:densityofsmoothfunctions}). Together these lemmas prove Lemma \ref{lem:propertiesofH1sigma}.
\begin{lemma}
Under the hypotheses of Theorem \ref{thm:strongtheorem} \eqref{eq:Poincareinequality} holds for all $u \in H^1(\Omega,\sigma)$. \label{lem:Poincareinequality}
\end{lemma}
\begin{proof}
Take $\rho(x) = \dist(x, \partial \Omega)$. Then Remark 5.3 of \cite{edmunds1993weighted} gives that the space $H^1(\Omega,\rho)$ has a Poincar{\'e} inequality. Note that Proposition \ref{prop:comparabletodistance3} guarantees that $H^1(\Omega, \sigma) = H^1(\Omega,\rho)$ as sets, and hence we have 
\begin{equation}
\int_\Omega (u(x) - (u)_\rho)^2 \rho(x)dx \leq C \int_\Omega |\nabla u(x)|^2\rho(x) dx\label{eq:generalPoincareinequality}
\end{equation}
for all $u \in H^1(\Omega,\sigma)$, where $(u)_\rho = \frac{\int_\Omega u(x) \rho(x) dx }{\int_\Omega \rho(x) dx}$. Now observe that
\begin{equation*}
\norm{u-(u)_\sigma}_{L^2(\Omega,\sigma)} = \min_{c \in \RR} \norm{u-c}_{L^2(\Omega,\sigma)}.
\end{equation*}  
Indeed, the right hand side is a $1$-D optimization problem with optimality condition
\begin{equation*}
c = \frac{\int_\Omega u(x) \sigma(x) dx}{\int_\Omega \sigma(x)},
\end{equation*}
the denominator being finite since $\sigma$ is a probability distribution. As such,
\begin{align*}
\norm{u-(u)_\sigma}_{L^2(\Omega,\sigma)}^2 &\leq \norm{u - (u)_\rho}^2_{L^2(\Omega, \sigma)},\\
&\leq C\norm{u - (u)_\rho}^2_{L^2(\Omega,\rho)},\\
&\leq C \norm{\nabla u}^2_{L^2(\Omega,\rho; \RR^d)},\\
&\leq C \norm{\nabla u}^2_{L^2(\Omega,\sigma; \RR^d)}.
\end{align*}
In the second line we used Proposition \ref{prop:comparabletodistance3}, in the third we used the Poincar{\'e} inequality for $H^1(\Omega,\rho)$, and in the fourth we used Proposition \ref{prop:comparabletodistance3} again. 
\end{proof}
\begin{lemma}
Under the hypotheses of Theorem \ref{thm:strongtheorem} the space $C^\infty(\overbar{\Omega})$ is dense in $H^1(\Omega,\sigma)$.
\label{lem:densityofsmoothfunctions}
\end{lemma}
\begin{proof}
Density of $C^\infty(\overbar{\Omega})$ in $H^1(\Omega,\rho)$, for $\rho(x):=\dist(x, \partial \Omega)$ and $\Omega$ having a Lipschitz boundary is shown, for example, in Theorem 7.2 of \cite{kufner1985weighted}. By Proposition \ref{prop:comparabletodistance3}, $\sigma$ is comparable to $\rho(x)$, and hence density of $C^\infty(\overbar{\Omega})$ in  $H^1(\Omega,\rho)$ carries over to $H^1(\Omega,\sigma)$. 
\end{proof}

\begin{proof}{\textbf{of Lemma \ref{lem:tildeJcoerciveandinGamma0}}}
Let $\tilde{G}$ be as given in \eqref{eq:tildeGdef}. For $u \in \bar{H}^1(\Omega,\sigma)$,
\begin{align*}
\frac{2}{\lambda}\tilde{G}(\nabla u) &= \int_\Omega (|\nabla u|^2 - 2|\nabla u| + 1) \sigma(x) dx - \int_{\{ |\nabla u| < 1\}} (|\nabla u|^2 - 2|\nabla u| +1 )\sigma(x) dx,\\
&\geq \int_\Omega |\nabla u|^2 \sigma(x) dx - 2\left(\int_\Omega |\nabla u|^2 \sigma(x) dx\right)^{1/2} -1.
\end{align*}
Since $H^1(\Omega,\sigma)$ has a Poincar{\'e} inequality, we obtain that for all $u \in \bar{H}^1(\Omega,\sigma)$,
\begin{equation*}
\tilde{G}(\nabla u) \geq C\normusigma^2 - \lambda\normusigma - \frac{\lambda}{2}.
\end{equation*}
Since $u \mapsto \langle u, g-f \rangle$ is a continuous linear functional on $H^1(\Omega,\sigma)$, we therefore obtain
\begin{equation*}
\tilde{J}(u) \geq C\normusigma^2- (\lambda+ C) \normusigma -\frac{\lambda}{2}.
\end{equation*}
Thus, $\tilde{J}$ is coercive on $\bar{H}^1(\Omega,\sigma)$.

$\tilde{J}$ is obviously proper. It is convex and continuous by Lemma \ref{lem:Gonweightedspaceconvexandcontinuous}. Weak lower semi-continuity then follows since $\tilde{J}$ is convex and continuous. 
\end{proof}

For clarity of presentation we have separated the portion of Theorem \ref{thm:strongtheorem} not implied by Theorem \ref{thm:weaktheorem} into Lemmas \ref{lem:existenceofsolutiontotildeGPlambda}, \ref{lem:equalityofsupandmax}, and \ref{lem:strongextremality}; these correspond to statements \ref{claim:tildegplambdahassolution}, \ref{claim:equivalenceofsupandmax}, and \ref{claim:strongextremality} of Theorem \ref{thm:strongtheorem}, respectively. 

\begin{lemma}
Under the assumptions of Theorem \ref{thm:strongtheorem}, \eqref{prob:tildeGPlambda} has a solution.
\label{lem:existenceofsolutiontotildeGPlambda}
\end{lemma}
\begin{proof}
Since $\tilde{J}$ (defined in \eqref{eq:tildeJdeff}) is proper and coercive (by Lemma \ref{lem:tildeJcoerciveandinGamma0}), the infimum in question is finite. Let $(u_n)_{n=1}^\infty \subset \bar{H}^1(\Omega,\sigma)$ be a minimizing sequence, i.e.
\begin{equation*}
\lim_{n \rightarrow \infty} \tilde{J}(u_n) = \inf_{u \in \bar{H}^1(\Omega,\sigma)} \tilde{J}(u_n).
\end{equation*}
Since this infimum is finite, coercivity of $\tilde{J}$ implies that $(u_n)_{n=1}^\infty$ is bounded in $\bar{H}^1(\Omega,\sigma)$. By Banach-Alaoglu there exists a weakly convergent subsequence converging to some $u_0 \in \bar{H}^1(\Omega,\sigma)$, and by weak lower semi-continuity of $\tilde{J}$, $u_0$ must be a minimizer of $\tilde{J}(u, \nabla u)$ over $H^1(\Omega,\sigma)$. Thus, $u_0$ is a solution to \eqref{prob:tildeGPlambda}.
\end{proof}

\begin{lemma}
\label{lem:equalityofsupandmax}
Under the assumptions of Theorem \ref{thm:strongtheorem}, 
\begin{equation}
\sup \eqref{prob:GPlambda} = \max \eqref{prob:tildeGPlambda}.\label{eq:equalityofinfandmin}
\end{equation}
\end{lemma}
\begin{proof}
 Note that the functionals in \eqref{prob:GPlambda} and \eqref{prob:tildeGPlambda} agree on $C^\infty(\overbar{\Omega})$ and are continuous on both $H^1(\Omega)$ and $\honesigma$, respectively. Since $C^\infty(\overbar{\Omega})$ is dense in both spaces, \eqref{eq:equalityofinfandmin} then follows.
\end{proof}
\begin{lemma}
Under the assumptions of Theorem \ref{thm:strongtheorem}, if $u_0$ solves \eqref{prob:tildeGPlambda} and $Q_0$ solves \eqref{prob:CPlambda}, then \eqref{eq:strongextremality} holds.\label{lem:strongextremality}
\end{lemma}
\begin{proof}
Evaluating inequality \eqref{eq:funccoercedbygrad} at $\xi^* = \ww_{Q_0}(x)$, and $\xi = \nabla u_0(x)$ and integrating, we get
\begin{equation*}
0 \geq \int_\Omega \frac{1}{2\lambda \norm{\sigma}_{L^\infty(\Omega)}}|\ww_{Q_0} + \lambda \sigma(x) (|\nabla u_0 (x)| - 1)_+ \frac{\nabla u_0(x)}{|\nabla u_0(x)|}|^2 dx,
\end{equation*}
whence the first equality of \eqref{eq:strongextremality} follows. To obtain the second, note that for almost all $x$ with $\ww_{Q_0}(x) \neq 0$, the first equality implies that
\begin{equation*}
\nabla u_0(x) = -\alpha(x) \ww_{Q_0}(x)
\end{equation*} 
for some $\alpha(x) >0$. Taking the magnitude of both sides of the first equality gives
\begin{equation*}
\alpha(x) = \frac{1}{\lambda \sigma(x)} + \frac{1}{|\ww_{Q_0}(x)|},
\end{equation*} 
as desired.

\end{proof}

\end{document}